\documentclass{article}

\usepackage[preprint]{iquest_med}

\setcitestyle{numbers,square,comma}

\usepackage[utf8]{inputenc}
\usepackage[T1]{fontenc}
\usepackage{amsmath}
\usepackage{amsfonts}
\usepackage{amssymb}
\usepackage{amsthm}
\usepackage{mathtools}
\usepackage{bm}
\usepackage{booktabs}
\usepackage{array}
\usepackage{tabularx}
\usepackage{microtype}
\usepackage{url}
\usepackage{graphicx}
\usepackage[most]{tcolorbox}
\usepackage{xcolor}
\usepackage{algorithm}
\usepackage{algpseudocode}
\usepackage{enumitem}
\usepackage{placeins}
\usepackage{longtable}

\definecolor{darkblue}{rgb}{0.0, 0.0, 0.55}
\definecolor{motivationbg}{RGB}{255,241,224}

\newtcolorbox{boeanchor}[1]{
  enhanced jigsaw,
  width=\linewidth,
  colback=IQuestDeepBlue!3!white,
  colframe=IQuestBlue,
  boxrule=0pt,
  leftrule=1.2pt,
  arc=0.8mm,
  outer arc=0.8mm,
  left=7pt,
  right=7pt,
  top=5pt,
  bottom=5pt,
  before skip=8pt plus 2pt minus 2pt,
  after skip=8pt plus 2pt minus 2pt,
  before upper={\textcolor{IQuestBlue}{\small\sffamily\bfseries #1}\par\smallskip}
}

\newtcolorbox{boetakeaway}{
  enhanced jigsaw,
  width=\linewidth,
  colback=black!1,
  colframe=IQuestBlue!65!black,
  boxrule=0pt,
  leftrule=0.8pt,
  sharp corners,
  left=6pt,
  right=4pt,
  top=3pt,
  bottom=3pt,
  before skip=5pt,
  after skip=7pt,
  before upper={\small\textcolor{IQuestBlue}{\sffamily\bfseries Exact takeaway. }\ignorespaces}
}

\usepackage{hyperref}
\hypersetup{
    colorlinks=true,
    linkcolor=darkblue,
    citecolor=darkblue,
    urlcolor=darkblue,
}

\theoremstyle{plain}
\newtheorem{theorem}{Theorem}[section]
\newtheorem{proposition}[theorem]{Proposition}

\newtheorem{corollary}[theorem]{Corollary}
\theoremstyle{definition}
\newtheorem{definition}[theorem]{Definition}
\newtheorem{assumption}[theorem]{Assumption}

\theoremstyle{remark}

\newcommand{\E}{\mathbb{E}}
\newcommand{\Prob}{\mathbb{P}}

\newcommand{\ind}{\mathbf{1}}

\newcommand{\MI}{\operatorname{I}}
\newcommand{\EVSI}{\operatorname{EVSI}}

\newcommand{\argmax}{\operatorname*{arg\,max}}

\newcolumntype{Y}{>{\raggedright\arraybackslash}X}

\title{Best-of-Evidence: Best-of-N Selection under Partial Verification}

\author{
  \textbf{
    Cenwei Zhang\affmark{1$\dagger$}
    \quad Teng Fang\affmark{1}
    \quad Yuxia Wang\affmark{2}
    \quad Derek Li\affmark{1}
    \quad Bryan Dai\affmark{1$\ddagger$}
    \quad Lei You\affmark{3$\ddagger$}
  }\\[0.6ex]
  \affmark{1}IQuest Research
  \quad
  \affmark{2}INSAIT
  \quad
  \affmark{3}Technical University of Denmark\\[0.6ex]
  \texttt{cwzhang2001@gmail.com}
  \quad
  \texttt{yuxia.wang@insait.ai}\\[0.4ex]
  \texttt{\{tfang,jdli,cbdai\}@iquestlab.com}
  \quad
  \texttt{leiyo@dtu.dk}\\[2.0ex]
  \affmark{$\dagger$}Work done during an internship at IQuest Research.
  \affmark{$\ddagger$}Corresponding authors.
}

\begin{document}

\maketitle

\begin{abstract}

BoN improves model outputs by sampling several candidates and selecting one with a proxy score, but it assumes that complete candidates can be evaluated reliably. Many vision-language tasks instead provide only partial verification: a finding, span, value, region, or relation may be checkable even when no dependable whole-response verifier exists. Moreover, the same claim may recur across candidates with opposing stances, allowing one observation to support part of the pool and contradict another. We introduce \emph{Best-of-Evidence} (BoE), an inference-time selection framework that keeps the BoN candidate pool fixed, represents reusable claims with a signed candidate--factor graph, and allocates a limited budget to evidence actions that can change the final choice. BoE formalizes selection under partial verification and provides a practical score-based controller, with the zero-budget case recovering the underlying BoN decision. Theoretically, we show that residual evidence capacity limits any evidence-driven improvement and that shared factor queries can achieve an $O(\log K)$ versus $\Theta(K)$ query separation in a factor-code model. Common-ledger experiments on four medical VQA settings show that BoE can improve fixed-pool selection and rescue some BoN failures when evidence is reliable, contrastive, and decision-relevant, while also revealing the channel-quality and candidate-generation limits that prevent universal gains.

\end{abstract}

\begin{figure}[h]
\centering
\includegraphics[width=\linewidth]{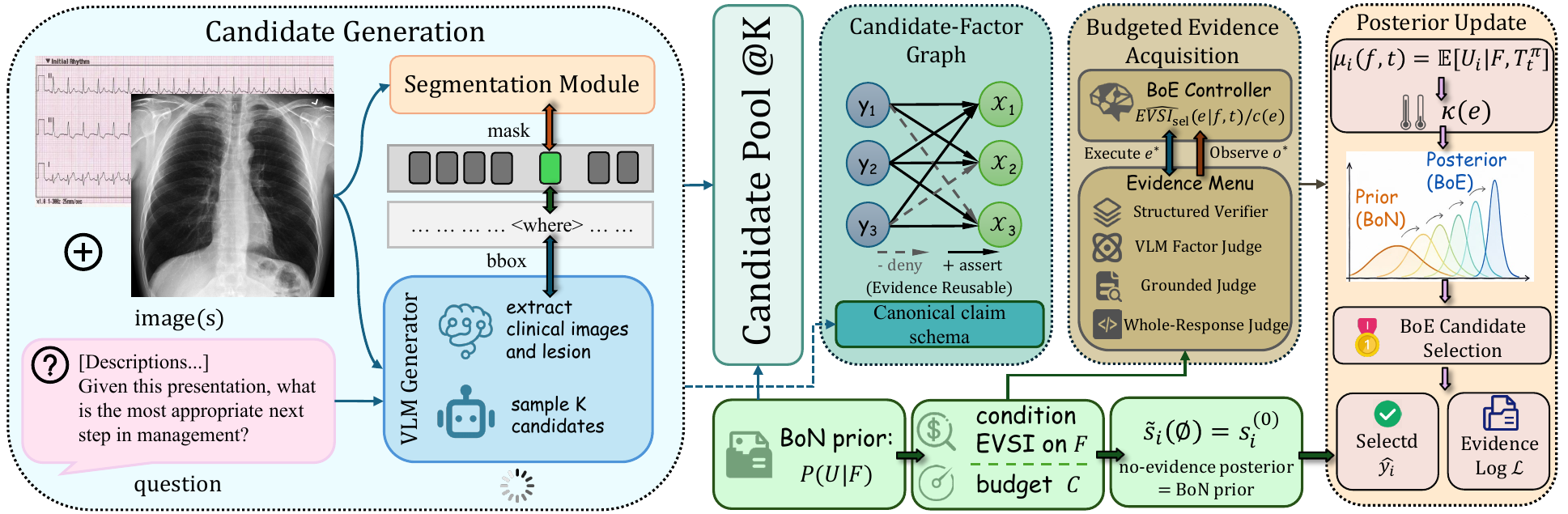}
\caption{\textbf{Overview of Best-of-Evidence under partial verification.}
From left to right, a VLM generates a fixed candidate pool, with segmentation
used only to prepare grounded views. The cheap BoN context initializes the
ranking. Candidate claims are merged into a signed candidate--factor graph,
allowing one observation to affect multiple candidates in different
directions. Under a budget, BoE selects among heterogeneous checks, updates
the ranking with the acquired evidence, and returns the selected candidate
together with an evidence log.}
\label{fig:boe_overview}
\end{figure}

\section{Introduction}

Language models often produce several plausible solutions even when a single
decode is unreliable. When changing or retraining the generator is costly,
sampling multiple attempts and selecting among them offers a simple way to
trade test-time compute for output quality
\citep{wang2023selfconsistencyimproveschainthought,cobbe2021trainingverifierssolvemath,
brown2024largelanguagemonkeysscaling,snell2024scalingllmtesttimecompute}.
\emph{Best-of-$N$} (BoN) follows this pattern: it samples $K$ candidates,
scores each with a proxy, and returns the highest-scoring one. The proxy may
be a majority vote, reward model, verifier, process reward model, or another
judge
\citep{lightman2023letsverifystepstep,zheng2023judgingLLM-as-a-Judge,
wang2025visualprm}. BoN is attractive because it leaves the generator fixed
and spends additional compute only on test-time generation and selection.

The limitation is the unit of verification. Standard outcome-level BoN
assumes that each complete candidate admits a reliable scalar score
\citep{lightman2023letsverifystepstep}. This is natural for code, arithmetic,
and some closed-form tasks
\citep{cobbe2021trainingverifierssolvemath,
chen2021evaluatinglargelanguagemodels,
chen2022codetcodegenerationgenerated}, but less natural for many
vision-language model (VLM) tasks. A candidate may be partly right and partly
wrong
\citep{wang2025visualprm,brown2024largelanguagemonkeysscaling,
jinnai2025regularized,chen2024unihd,wang2026evpv}. A report may contain a
correct modality claim and an unsupported finding
\citep{miura-etal-2021-improving,delbrouck-etal-2022-improving,
heiman2025factchexcker}; a chart explanation may read one value correctly but
compare two values incorrectly \citep{masry2022chartqa,liu2023deplot}; and a
document answer may depend on multiple OCR spans and a layout relation
\citep{mathew2021docvqa,mathew2021infographicvqa}. In such cases, no cheap
whole-response verdict may exist, although a crop, lookup, region, span, or
isolated claim can still be checked
\citep{chen2024unihd,surismenon2023vipergpt,
gupta2022visualprogramming,ding2025armthinker}. Selection therefore faces a
granularity mismatch: candidates are complete objects, while verification
arrives through costed local views.

\emph{Best-of-Evidence} (BoE) addresses this mismatch while keeping the
sampled pool fixed. The cheap BoN context initializes candidate ranking; BoE
then represents canonical claims and candidate stances, and spends a limited
budget on checks that may change the final choice. In an ideal Bayesian
formulation, observations update posterior mean utilities. The practical
controller approximates this objective with discrimination-weighted scores
and a plug-in value-of-information rule. With zero evidence budget, BoE
recovers the underlying BoN decision.

This view builds on recent reward-evaluation systems that organize
heterogeneous tools, references, checklists, and verifiers into grounded
evaluation procedures
\citep{chen2026skillrm,hu2026openreward,zhang2026agentvr}. These systems show
how evidence can be gathered and aggregated; the remaining question here is
which checks are worth acquiring when no dependable whole-response score is
available. A lookup, local judge, or whole-response evaluator is useful only
insofar as its possible outcomes can alter selection.

BoE represents that selection structure with a signed candidate--factor graph.
Factor nodes are inspectable claims, and edges record whether each candidate
asserts, denies, or is unrelated to a claim \citep{K2001factorgraphs}. One
observation can therefore support part of the pool and contradict another.
Useful reuse is not determined by graph degree alone: the checked claim must
separate candidates that remain competitive under the cheap prior. This
signed sharing permits compressive verification
\citep{Aldridge2026Grouptesting} without assuming that every candidate pool
forms an exact latent factor code.

\paragraph{We make three contributions in this paper:}
\begin{itemize}[leftmargin=1.4em]
    \item We formulate selection under partial verification and introduce a
    budgeted score-based controller that reuses signed local evidence across a
    fixed BoN candidate pool.
    \item We establish a residual-information upper bound, a scoped
    factor-code query separation, and a local account of decision-effective
    evidence reuse.
    \item We evaluate BoE with common evidence ledgers on four medical VQA
    settings, characterizing fixed-ledger policy differences and the channel
    and candidate-pool limits of partial verification.
\end{itemize}
\section{Selection under Partial Verification}
\label{sec:problem}

Before introducing BoE, we first define the partial-verification selection
problem it approximates. We proceed from the fixed candidate pool, to the
cheap prior, to shared verifiable claims, and finally to the budgeted
selection objective that Section~\ref{sec:boe_method} seeks to approximate.

\paragraph{Notation conventions.}
Random objects are uppercase and their realizations are lowercase when the
distinction matters. We use $K$ for the candidate-pool size, $U_i$ for latent
candidate utility, $\mu_i$ for an exact Bayesian posterior mean, and
$\widetilde s_i$ for the uncalibrated score used in ledger replay. The symbol
$\pi$ denotes an evidence-acquisition policy, whereas
$p_\phi^{\mathrm{gen}}$ denotes the candidate generator. Candidate, claim,
acquisition-round, dataset, and question indices are $i,j,t,d,r$; $m$ denotes
a terminal selection size and $q$ a query count. We use $E_t$ and $O_t$ for
random acquired actions and outcomes, $e$ and $o$ for their realizations,
$\mathsf t$ for a realized transcript, and $O_e$ for the prospective outcome
of acquiring action $e$ next. Unless explicitly stated in bits, mutual
information is measured in nats.

\subsection{Candidate pool and cheap prior}

We first fix the objects among which selection is performed. Let $(\Omega,X)$
denote a random task instance. The record $X$ contains the multimodal input,
question, and admissible same-instance metadata, but not the benchmark
reference. Conditional on $X=x\in\mathcal X$, the generator samples the joint
candidate pool
\begin{equation}
\label{eq:candidate_sampling}
\mathbf{Y}_{1:K}\mid X=x
\sim
p_\phi^{\mathrm{gen}}(\cdot\mid x).
\end{equation}
This equation fixes the pool used by all later verification and selection.
A realized pool $\mathbf y_{1:K}$ remains unchanged throughout, and no
conditional independence among its candidates is required.

A candidate may be a short answer, report, box set, mask, or another structured
output. The latent state $\Omega$ contains the facts or reference information
needed to define correctness, but neither the selector nor an evidence action
may query it. Candidate utility is
\begin{equation}
\label{eq:utility}
U_i=u(\Omega,\mathbf{Y}_i)\in[0,1],
\qquad
\mathbf{U}=(U_1,\ldots,U_K).
\end{equation}
This defines the hidden target that evidence and selection seek to infer.

Before purchasing evidence, the selector observes only
$F=f_0(X,\mathbf Y_{1:K})$. The fixed cheap interface $f_0$ may expose
candidate outputs, answer frequencies, validity flags, self-consistency
statistics, and other declared zero-cost features, but not the complete record
$X$. It induces
\begin{equation}
\label{eq:prior}
P(\mathbf{U}\mid F),
\qquad
\mu_i^{\mathrm{prior}}(F)=\E[U_i\mid F].
\end{equation}
This is the zero-evidence Bayesian ranking. The problem is nontrivial only when
$\mathbf U$ is not almost surely determined by $F$. An ideal Bayesian BoN
selector maximizes $\mu_i^{\mathrm{prior}}(F)$ with fixed tie-breaking. The
ledger implementation instead initializes an answer-frequency score
$s_i^{(0)}$ whose maximizer agrees with raw majority; it is not assumed that
$s_i^{(0)}=\mu_i^{\mathrm{prior}}(F)$.

\subsection{Signed shared claims}

We next define what partial verification can address. Let
$\chi_1,\ldots,\chi_M$ be canonical verifiable propositions about $\Omega$,
called factors when represented as graph nodes, and let
$\boldsymbol{\Theta}=(\Theta_1,\ldots,\Theta_M)\in\{0,1\}^M$
denote their latent truth values. Examples include a chart value, OCR span,
spatial relation, modality, anatomical site, or visual finding.

The signed incidence matrix
\begin{equation}
\label{eq:signed_incidence}
\mathbf{B}\in\{-1,0,+1\}^{K\times M}
\end{equation}
records candidate stance: $B_{ij}=+1$ if candidate $i$ asserts $\chi_j$,
$B_{ij}=-1$ if it denies $\chi_j$, and $B_{ij}=0$ if the claim is irrelevant.
The same information is represented as the signed candidate--factor graph
\begin{equation}
\label{eq:graph}
G_{\mathrm{cf}}
=
(\mathcal{V}_c,\mathcal{V}_f,\mathcal{E}_{\mathrm{cf}}),
\end{equation}
where $\mathcal V_c$ contains candidate nodes, $\mathcal V_f$ contains claim
nodes, and signed edges correspond to nonzero entries of $\mathbf B$.

Together, these equations determine how one observation is reused. Confirming
claim $j$ supports candidates with $B_{ij}=+1$, contradicts those with
$B_{ij}=-1$, and leaves those with $B_{ij}=0$ unchanged; rejection reverses
the two nonzero directions. The observation is acquired once and applied
through all incident edges.

The graph is an evidence interface, not a complete causal or generative model
of utility. It need not explain every component of $U_i$ or require
$\boldsymbol\Theta$ to determine $\mathbf U$. Sharing alone is also
insufficient: if all competitive candidates take the same stance, a check may
move their scores together without changing the winner. Useful shared evidence
must therefore be residual and contrastive under the cheap prior.

\subsection{Costed evidence and the decision objective}

We finally define how additional information is acquired. Let $\mathcal A$
denote the evidence-action menu. At round $t$, let $E_t$ be the random action,
$O_t$ its random outcome, and
$T_{t-1}^{\pi}=(E_1,O_1,\ldots,E_{t-1},O_{t-1})$
the purchased transcript prefix. The executor constructs an action-specific
view of the same record and returns an observation:
\begin{equation}
\label{eq:access_channel}
Z_t=\operatorname{view}_{E_t}(X,\mathbf Y_{1:K},T_{t-1}^{\pi}),
\qquad
O_t=\psi_{E_t}(Z_t,F,T_{t-1}^{\pi};
\mathsf{Seed}_t^{\mathrm{exec}}).
\end{equation}
The equation separates what an action is allowed to inspect from the verdict
it returns. A view may be a crop, mask-guided region, metadata field,
structured lookup, or human inspection of the same $X$. Neither map may query
$\Omega$, $\mathbf U$, or the benchmark reference. Factor actions inspect
claim nodes, whereas candidate actions inspect complete responses.
Localization only determines the view; it is not an additional observation.

An access-admissible policy observes only $F$, purchased action--outcome pairs,
and independent private randomization. It produces a terminal transcript
$T_\pi$ with realized cost at most $C$. The complete measurability,
stopping-time, and transcript definitions are given in
Appendix~\ref{app:formal_problem}; let $\Pi_C^{\mathrm{acc}}$ denote this
policy class.

After observing $T_\pi$, the Bayes action selects the candidate with the
largest conditional mean utility. The ideal budgeted value is
\begin{equation}
\label{eq:policy_objective}
\operatorname{OPT}_{C}^{\mathrm{sel}}
=
\sup_{\pi\in\Pi_C^{\mathrm{acc}}}
\underbrace{
\E\!\left[
\max_{1\le i\le K}\E[U_i\mid F,T_\pi]
\right]
}_{\substack{\text{expected utility of the candidate selected}\\
\text{after the purchased evidence}}}
.
\end{equation}
Equation~\eqref{eq:policy_objective} is the key to Section~\ref{sec:boe_method}. For a fixed evidence policy, the inner
conditional expectation gives the posterior mean utility of each candidate
after the purchased observations, and the maximum selects the best candidate
under that updated belief. The outer expectation averages over task
instances, candidate pools, evidence outcomes, and policy randomness. The
supremum asks which budget-feasible evidence policy gives the best final
selection on average.

With $C=0$, the transcript is empty and the objective reduces to ideal
Bayesian BoN. Exact optimization is generally intractable because actions have
heterogeneous costs, correlated effects, and adaptive availability.
Section~\ref{sec:boe_method} therefore uses a sequential
value-of-information approximation.
\section{Best-of-Evidence}
\label{sec:boe_method}

BoE is a practical sequential approximation to
Equation~\eqref{eq:policy_objective}. As illustrated in
Figure~\ref{fig:boe_overview}, it builds signed shared claims from a fixed
candidate pool, attaches heterogeneous checks, allocates a budget according to
their predicted effect on selection, and re-ranks the unchanged pool after
each observation.

\subsection{Bayesian evidence allocation}

BoE canonicalizes equivalent claims, preserves candidate stance, and
associates each claim with admissible evidence actions. The menu may combine
structured verification, whole-image or grounded VLM judgments, and
whole-response judging. An unresolved or inapplicable claim produces no
applicable signal rather than a contradiction.

For a realized cheap context $F=f$ and transcript
$\mathsf t=((e_1,o_1),\ldots,(e_t,o_t))$, the ideal posterior mean utility is
\begin{equation}
\label{eq:posterior_utility}
\mu_i(f,\mathsf t)
=
\E[U_i\mid F=f,T_t^\pi=\mathsf t].
\end{equation}
Let
$V_{\mathrm{sel}}(f,\mathsf t)=\max_i\mu_i(f,\mathsf t)$
denote the value of selecting immediately. For a feasible unrevealed action
$e$, its exact expected value of sample information is
\begin{equation}
\label{eq:evsi}
\EVSI_{\mathrm{sel}}(e\mid f,\mathsf t)
=
\E_{O_e\mid F=f,T_t^\pi=\mathsf t}
\!\left[
V_{\mathrm{sel}}\!\left(f,\mathsf t\oplus(e,O_e)\right)
\right]
-
V_{\mathrm{sel}}(f,\mathsf t).
\end{equation}
This quantity measures the expected change in final selection value rather
than claim accuracy or graph degree alone. BoE applies this principle
sequentially: it acquires a feasible action with high estimated value per
cost, updates the transcript and candidate scores, and stops when the
remaining budget or estimated value no longer justifies another check. The
formal greedy rule, stopping condition, and complete replay procedure are
given in Appendix~\ref{app:boe_controller} and
Algorithm~\ref{alg:boe}.

\subsection{Discrimination-weighted ledger replay}

The benchmark replay does not fit a complete observation model for
$P(\Omega,\mathbf U,T_\pi\mid F)$. It initializes
$\widetilde s_i(\varnothing)=s_i^{(0)}$ and maps each outcome
$o\in\mathcal O_e$ to a signed candidate effect
$\alpha_{ie}(o)\in\{-1,0,+1\}$. Each action also inherits a pooled outcome
mass $\widehat p_e$ and a selection-side discrimination index
$\kappa(e)\in(0,1)$. Writing
$\operatorname{wt}(e)=\operatorname{logit}\kappa(e)$, the implemented update is
\begin{equation}
\label{eq:score_update}
\operatorname{logit}\widetilde s_i\!\left(\mathsf t\oplus(e,o)\right)
=
\operatorname{logit}\widetilde s_i(\mathsf t)
+\alpha_{ie}(o)\operatorname{wt}(e).
\end{equation}
Thus $\alpha_{ie}(o)=0$ or $\kappa(e)=0.5$ gives no update, while
$\kappa(e)<0.5$ reverses an anti-correlated channel. The controller evaluates
all possible outcomes with the following plug-in score.

\begin{boeanchor}{Implemented replay action value}
\begin{equation}
\label{eq:plugin_evsi}
\widehat{\EVSI}_{\mathrm{sel}}(e\mid f,\mathsf t)
=
\underbrace{
\sum_{o\in\mathcal O_e}
\widehat p_e(o)
\max_i\widetilde s_i\!\left(\mathsf t\oplus(e,o)\right)
}_{\substack{\text{expected best score}\\\text{after check }e}}
-
\underbrace{
\max_i\widetilde s_i(\mathsf t)
}_{\text{current best score}}.
\end{equation}
\small
Feasible actions are ranked by
$\widehat{\EVSI}_{\mathrm{sel}}(e\mid f,\mathsf t)/c(e)$. The hat marks a
ledger-instantiated score approximation rather than exact Bayesian EVSI.
\end{boeanchor}

The pooled $\widehat p_e$ is not a history-conditioned predictive model, and
$\kappa(e)$ measures candidate discrimination rather than factor-truth
accuracy or a likelihood parameter. Accordingly, $\widetilde s_i$ is an
evidence-updated selection score, not a calibrated posterior probability.
Appendix~\ref{app:boe_controller} gives the full score construction and its
relation to the exact Bayesian quantities. The evidence log records each acquired action, source, outcome, cost, discrimination index, and affected candidates. It contains no chain-of-thought text.


\section{When Can Shared Evidence Help?}
\label{sec:theory}

Shared evidence can help only if it is decision-relevant, structurally reusable,
and informative beyond the cheap context. We summarize these three mechanisms
and defer their formal statements and proofs to Appendix~\ref{app:formal_theory}.

\subsection{Decision-effective signed reuse}

At an active history $h=(f,\mathsf t)$, let $\mathbf b_j$ be column $j$ of the
signed candidate--claim matrix. A positive semidefinite matrix $\mathbf W_h$ weights score directions by their effect on
the current selection.
We define
\begin{equation}
\label{eq:effective_reuse_main}
R_{\mathrm{eff}}(j;h)
=
\mathbf b_j^\top\mathbf W_h\mathbf b_j.
\end{equation}
Thus, a claim can touch many candidates yet have little value if it moves them
in a common direction.

\begin{boeanchor}{When is a shared claim worth checking?}
{\small For a factor action $e$ targeting claim $j(e)$, let $\iota_e(h)$ be
the conditional variance of its centered score innovation and $c(e)$ its
cost. The local model yields, up to a common factor $1/2$,\par}
\begin{equation}
\label{eq:factor_density_main}
D_{\mathrm{factor}}(e;h)
=
\frac{
\underbrace{\iota_e(h)}_{\substack{\text{unpredictable}\\\text{score signal}}}
\underbrace{R_{\mathrm{eff}}(j(e);h)}_{
\substack{\text{signed decision}\\\text{contrast}}}
}{
\underbrace{c(e)}_{\text{cost}}
}.
\end{equation}
\small
This is a local explanatory proxy, not a Taylor expansion or the implemented
controller objective. The replay controller does not estimate $\mathbf W_h$,
$R_{\mathrm{eff}}$, or $\iota_e(h)$.
\end{boeanchor}

The formula explains why useful reuse requires signal, signed contrast, and
favorable cost rather than graph degree alone. Appendix~\ref{app:formal_theory}
compares this density with candidate-level checks.

\begin{figure}[H]
\centering
\includegraphics[width=\linewidth]{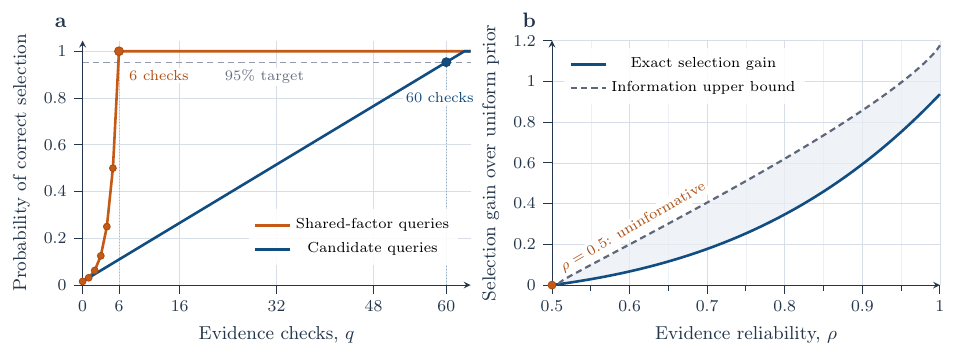}
\caption{\textbf{Exact consequences of the finite constructions.}
\textbf{a}, Six shared-factor queries recover a six-bit code; candidate queries
need 60 checks to exceed $95\%$ success.
\textbf{b}, In the four-factor noisy model, the exact gain
$\rho^4-1/16$ stays below the information upper bound and vanishes with an
uninformative channel.}
\label{fig:exact_theory_checks}
\end{figure}

\subsection{A scoped query-complexity separation}
\label{sec:controlled_theory_probes}

We prove a structural witness in which $K=2^M$ candidates encode all
assignments of $M$ binary claims. Exact factor queries reveal shared
coordinates, whereas exact candidate queries test complete assignments.

\begin{boetakeaway}
For target success at least $1-\varepsilon$, $\varepsilon\in[0,1)$,
\[
\underbrace{Q_{\mathrm{factor}}=\log_2 K}_{\text{shared-coordinate checks}}
\qquad\text{whereas}\qquad
\underbrace{Q_{\mathrm{candidate}}\ge
\left\lceil(1-\varepsilon)K\right\rceil-1}_{\text{whole-candidate checks}}.
\]
\end{boetakeaway}

This proves a possible $O(\log K)$ versus $\Theta(K)$ gap, not that arbitrary
VLM pools form complete codes. Figure~\ref{fig:exact_theory_checks} shows this
construction and a separate noisy-factor model; the closed-form calculations
and access-admissible proof appear in Appendix~\ref{app:formal_theory}. The
noisy-model parameter $\rho$ is distinct from the empirical replay index
$\kappa$.

\subsection{Residual-information limit}

For an adaptive policy $\pi$, let $T_\pi$ contain its acquired actions and
observations. The residual evidence capacity under budget $C$ is
\begin{equation}
\label{eq:evidence_capacity_main}
\Lambda_C
=
\sup_{\pi\in\Pi_C^{\mathrm{acc}}}
\MI(\mathbf U;T_\pi\mid F).
\end{equation}
It removes information already present in $F$ and accounts for redundancy
among adaptive observations.

Under the conditional sub-Gaussian assumption in Appendix~\ref{app:formal_theory},
we prove the following evidence-only ceiling for a terminal selection of size
$m$.

\begin{boeanchor}{The information ceiling}
\begin{equation}
\label{eq:main_evidence_bound}
\E[V_m^\pi(F,T_\pi)-V_m(F)]
\le
\underbrace{
\sqrt{\frac{m}{2}\MI(\mathbf U;T_\pi\mid F)}
}_{\substack{\text{from acquired information}\\\text{beyond }F}}
\le
\underbrace{
\sqrt{\frac{m}{2}\Lambda_C}
}_{\substack{\text{from best accessible evidence}\\\text{within budget}}}.
\end{equation}
\small
Small residual information rules out a large evidence-driven improvement.
\end{boeanchor}

We also prove a total-throughput upper law combining cheap-context information
and $\Lambda_C$. These are necessary information constraints, not
achievability guarantees or controller objectives; all assumptions and
proofs are given in Appendix~\ref{app:formal_theory}.

\section{Experiments}
\label{sec:experiments}

We evaluate whether budgeted evidence improves selection after the candidate
pool and all potential observations are fixed. This common-ledger design
isolates the selection mechanism rather than end-to-end variation from new
candidate samples. It also probes the account in Section~\ref{sec:theory}:
useful gains require residual, decision-relevant evidence and a correct
candidate that selection can recover. Detailed protocols, the historical
SLAKE pilot, and other information appear in
Appendix~\ref{app:experimental_details}.

\subsection{Benchmark design}
\label{sec:experimental_setup}

\paragraph{Evaluation cells and models.}
The evaluation covers VQA-Med, PathVQA, an option-hidden PMC-VQA protocol, and
MedXpertQA-MM
\citep{benabacha2019vqamed,he2021pathvqa,zhang2024pmcvqa,zuo2025medxpertqa}.
PMC-VQA hides the original options and uses the correct option text as the
open-answer reference; MedXpertQA-MM retains five options and one to six
images. All four cells use Qwen3-VL-30B-A3B as the candidate generator and
Qwen3-VL-235B-A22B as the evidence judge. We sample $K=16$ candidates at
temperature $1.1$ and replay budgets $C\in\{1,2,4,8,16\}$, with $C=16$ as the
main comparison. Table~\ref{tab:experiment_cells} summarizes the resulting
candidate pools and their selection ceilings.

\begin{table}[H]
\centering
\small
\caption{\textbf{Main evaluation cells and candidate-generation health.}
Parsed is the fraction of generations with a recovered structured output;
Answer is the fraction with a usable final answer. Oracle@$K$ is the fraction
of questions whose candidate pool contains at least one correct answer.}
\label{tab:experiment_cells}
\resizebox{\linewidth}{!}{
\begin{tabular}{lllrrrr}
\toprule
Dataset & Protocol & Generator / judge
& $n_d$ & Parsed (\%) & Answer (\%) & Oracle@$K$ (\%) \\
\midrule
VQA-Med & Open answer, fixed battery & 30B / 235B
& 2,334 & 95.1 & 97.7 & 82.1 \\
PathVQA & Open answer, fixed battery & 30B / 235B
& 9,903 & 94.0 & 99.1 & 61.5 \\
PMC-VQA & Option-hidden open protocol & 30B / 235B
& 10,000 & 88.2 & 98.5 & 65.1 \\
MedXpertQA-MM & Single-answer five-way MCQ, 1--6 images & 30B / 235B
& 2,000 & 96.2 & 94.5 & 64.7 \\
\bottomrule
\end{tabular}}
\end{table}

\paragraph{Ledger, evidence, and policies.}
For each question, we cache one candidate pool, its signed factor graph, and
all available evidence outcomes, channel identities, costs, and evaluation
utilities. Replay policies therefore share the same generation and potential
observations; they differ only in which entries they reveal and how they
select from the updated scores. The evidence judge returns observations but
does not rank candidates.

Each open-answer candidate instantiates five slots: modality, anatomical
region, view or plane, primary finding, and one answer-relevant attribute.
Unresolved or inapplicable slots return $\varnothing$, and canonicalized claims
are merged with stance preserved. MedXpertQA-MM additionally records the
figure identifier. Structured checks have cost $1$, ordinary or grounded VLM
factor judgments cost $4$, and whole-response judgments cost $8$. These are
design costs rather than measured wall-clock or token costs.

At $C=16$, we compare raw BoN majority, random factor acquisition,
whole-response judging, BoE, and a myopic label-guided allocator. The last
uses evaluation labels to choose factors greedily and is a non-deployable,
factor-only diagnostic rather than an upper bound. Random acquisition is also
factor-only, while BoE may use whole-response checks. Their main comparison
therefore uses unmatched menus; a matched replay is reported separately.

\begin{figure}[t]
\centering
\includegraphics[width=\linewidth]
{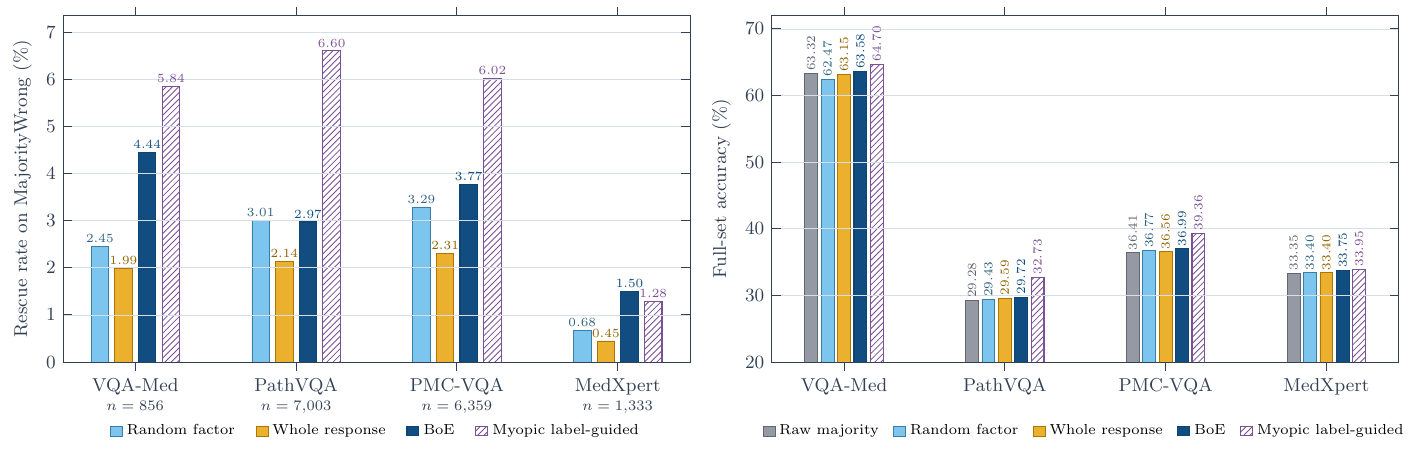}
\caption{\textbf{Selection performance at $C=16$.}
\textbf{Left}, rescue rates on the policy-aligned
\textsc{MajorityWrong} subsets; raw BoN is omitted because its accuracy is
zero by definition, and subset sizes are shown below the dataset names.
\textbf{Right}, full-set selected-candidate accuracy; the vertical axis begins
at $20\%$. In both panels, the myopic label-guided allocator is a
non-deployable, factor-only diagnostic with no upper-bound interpretation.
BoE value labels are bolded for visual emphasis.}
\label{fig:selection_summary}
\end{figure}

\paragraph{Evaluation metrics.}
Each dataset--channel group is assigned a discrimination index $\kappa_g$,
and its categorical outcome frequencies provide the plug-in mass
$\widehat p_e$ in Equation~\eqref{eq:plugin_evsi}. Fitted indices, outcome
frequencies, and policy evaluation use the same ledger. Open-answer
correctness combines deterministic matching with semantic-equivalence grading,
whereas MedXpertQA-MM uses exact matching over A--E. Because the open-answer
evaluator and evidence judge use the same model tier, whole-response and
channel-quality results are interpreted as mechanism diagnostics.

We report full-set selected-candidate accuracy, Oracle@$K$, and
\textsc{MajorityWrong}, the questions on which raw BoN selects an incorrect
candidate. Raw BoN has zero accuracy on this subset, while Oracle@$K$
distinguishes selection-fixable cases from generation failures. Paired
intervals and McNemar tests are conditional summaries of the retained ledgers
and one candidate-generation seed. Besides, we do not estimate $\Lambda_C$ on the VLM ledgers. For dataset $d$ with $n_d$
questions, policy $\pi$, budget $C$, and realized transcript
$\mathsf t_{r,C}^{\pi}$ for question $r$, we report the entropy-shaped score
proxy
\begin{equation}
\label{eq:sharpening_score}
\mathsf{Sharp}_{d,C}^{\pi}
=
\frac{1}{n_d}
\sum_{r=1}^{n_d}
\sum_{i=1}^{K}
\left[
H_b(s_{ri}^{(0)})
-H_b\!\left(\widetilde s_{ri}(\mathsf t_{r,C}^{\pi})\right)
\right].
\end{equation}
Here $H_b(x)=-x\log_2x-(1-x)\log_2(1-x)$. Because
$\widetilde s_{ri}$ is uncalibrated, $H_b$ is only a shape function. The
resulting dimensionless mean can be negative. It is neither information nor an
estimate, upper bound, or lower bound for $\Lambda_C$.

\subsection{Benchmark results}
\label{sec:benchmark_results}

Figure~\ref{fig:selection_summary} gives two views of the fixed-ledger result
at $C=16$. The right panel evaluates all questions, while the left panel
restricts attention to failures of the cheap BoN decision.

\paragraph{Full-set selection.}
BoE is $0.26$--$0.58$ percentage points above raw majority across the four
cells. This consistent direction shows that evidence-aware reranking can
improve a fixed candidate pool, but the magnitude is modest and does not by
itself identify the allocation effect: random and whole-response policies also
acquire evidence.

The closest policy comparison is therefore BoE against random factor
acquisition. Table~\ref{tab:paired_uncertainty} shows that only the VQA-Med
BoE--random interval excludes zero. PathVQA and PMC-VQA have positive
BoE--raw intervals but no separated BoE--random contrast, suggesting that
evidence can improve selection there without establishing an advantage for
the current allocation rule.

\begin{table}[H]
\centering
\small
\caption{\textbf{Question-paired uncertainty at $C=16$.}
Intervals and gaps are in percentage points. Discordants are counts of
BoE-correct/random-wrong and BoE-wrong/random-correct questions. McNemar
$p$-values refer only to BoE versus random factor. All analyses are exploratory
and uncorrected.}
\label{tab:paired_uncertainty}
\resizebox{\linewidth}{!}{
\begin{tabular}{lcccc}
\toprule
Dataset & BoE--raw (95\% interval) & BoE--random (95\% interval)
& Discordants & McNemar $p$ \\
\midrule
VQA-Med       & $+0.26\;[-0.47,\,+0.99]$ & $\mathbf{+1.11\;[+0.43,\,+1.80]}$ & 47 / 21  & 0.0022 \\
PathVQA       & $+0.43\;[+0.05,\,+0.81]$ & $+0.29\;[-0.09,\,+0.68]$ & 209 / 180 & 0.16 \\
PMC-VQA       & $+0.58\;[+0.19,\,+0.97]$ & $+0.22\;[-0.18,\,+0.61]$ & 213 / 191 & 0.30 \\
MedXpertQA-MM & $+0.40\;[-0.15,\,+0.95]$ & $+0.35\;[-0.20,\,+0.90]$ & 20 / 13 & 0.30 \\
\bottomrule
\end{tabular}}
\end{table}


On VQA-Med, the separated BoE--random contrast is $+1.11$ points, while the
smaller BoE--raw interval includes zero. Because random factor acquisition is
below raw majority and uses a narrower menu, this is a fixed-ledger policy
contrast rather than a demonstrated absolute or pure routing gain. On
MedXpertQA-MM, both paired intervals include zero, and restricting BoE to the
same factor-only menu reduces the descriptive gap from $+0.35$ to $+0.15$
points.

\begin{table}[H]
\centering
\small
\caption{\textbf{Assigned channel discrimination and entropy-shaped score
proxy at $C=16$.}
Each $\kappa_g$ is for the dataset--channel group specified by its row and
column; Structured $\kappa_g$ gives the range over populated fixed-battery
slots. $\mathsf{Sharp}_{d,C}^{\pi}$ is dimensionless and is not information
or $\Lambda_C$.}
\label{tab:evidence_diagnostics}
\resizebox{\linewidth}{!}{
\begin{tabular}{lrrrrrrr}
\toprule
Dataset
& VLM $\kappa_g$
& Grounded $\kappa_g$
& Structured $\kappa_g$
& Whole $\kappa_g$
& $\mathsf{Sharp}_{d,C}^{\mathrm{BoE}}$
& $\mathsf{Sharp}_{d,C}^{\mathrm{random}}$
& $\mathsf{Sharp}_{d,C}^{\mathrm{whole}}$ \\
\midrule
VQA-Med
& 0.539 & 0.505 & 0.522--0.624 & 0.715
& +1.025 & +0.280 & +0.653 \\
PathVQA
& 0.518 & 0.534 & 0.275--0.517 & 0.608
& $-0.737$ & +0.092 & $-0.617$ \\
PMC-VQA
& 0.502 & 0.522 & 0.377--0.534 & 0.623
& $-0.276$ & $-0.006$ & $-0.202$ \\
MedXpertQA-MM
& 0.514 & 0.526 & n/a & 0.642
& +1.006 & +0.021 & +0.799 \\
\bottomrule
\end{tabular}}
\end{table}

\paragraph{Selection on raw-majority failures.}
The left panel of Figure~\ref{fig:selection_summary} focuses on examples where
additional evidence has room to change the BoN decision. VQA-Med gives the
clearest result: BoE rescues $4.44\%$ of raw failures, compared with $2.45\%$
for random factor acquisition. PMC-VQA shows a smaller separation, while
PathVQA is effectively tied with random acquisition.

Candidate-pool coverage limits these rescue rates. On MedXpertQA-MM, only 627
of 1,333 raw failures contain a correct candidate; the remaining 706 cannot be
repaired by any selector. This supports the theoretical distinction between
available evidence and generation headroom: evidence can re-rank candidates,
but it cannot create a missing answer. Integer transition counts and
denominator checks appear in Appendix~\ref{app:experimental_details}.

\paragraph{Channel discrimination and the score proxy.}
Table~\ref{tab:evidence_diagnostics} relates the selection results to channel
quality. VQA-Med has the strongest populated structured channel and positive
BoE sharpening. PathVQA and PMC-VQA instead have VLM indices near $0.5$, some
anti-correlated structured slots, and negative realized BoE proxy values.
This pattern is qualitatively consistent with decision-effective reuse:
evidence is useful when its channel is discriminative and its signed update
can separate competitive candidates.

MedXpertQA-MM also shows why score sharpening is not sufficient for accuracy:
its BoE and whole-response proxies are positive, yet their full-set changes
remain small. Evidence must not only move scores, but move them in a
decision-relevant direction on a pool containing a correct answer. Overall,
the experiments support BoE as a fixed-pool evidence controller under
favorable channels, while the remaining cells expose the weak-evidence,
action-menu, and candidate-generation regimes predicted by the formulation.

\FloatBarrier
\section{Related Work}
\label{sec:related_work}

\paragraph{Candidate-level test-time selection.}
BoN samples several responses and selects one with a proxy\citep{cobbe2021trainingverifierssolvemath, brown2024largelanguagemonkeysscaling}. Process reward models add step-level supervision \citep{lightman2023letsverifystepstep}, while regularized BoN and Best-of-Tails control reward hacking and extreme-score errors \citep{jinnai2025regularized,hsu2026bestoftails}. Other work allocates generation or verification compute adaptively and replaces pointwise scores with pairwise or criteria-decomposed judgments \citep{snell2024scalingllmtesttimecompute,tang2025carbon,liu2025pairjudgermperformbestofn,kwok2026llmverifier,dughmi2026adap}. Their verification signal remains attached to a candidate or reasoning step. BoE treats a local claim shared by candidates with opposing stances as the verification unit.

\paragraph{Grounded verification and budgeted acquisition.}
VLM inference combines sampling with visual search and verification
\citep{jeddi2026avis}. Tools expose local evidence through cropping,
detection, retrieval, or execution
\citep{surismenon2023vipergpt,gupta2022visualprogramming,chen2024unihd};
VisualPRM, EVPV, and TIM-PRM score reasoning or verify visual premises
\citep{wang2025visualprm,wang2026evpv,kuang2025timprm}. ARM-Thinker,
Skill-RM, OpenReward, and AgentVR organize heterogeneous evaluation resources
\citep{ding2025armthinker,chen2026skillrm,hu2026openreward,
zhang2026agentvr}, while visual self-verification remains weak
\citep{wu2026ahamoment}. Value-of-information, active feature acquisition, and
submodular methods select costly observations
\citep{jackson2019valueinformation,li2021activefeature,
wei2015submodularity,krause2014submodular}; group testing studies shared
queries \citep{Aldridge2026Grouptesting}, factor graphs encode shared
structure \citep{K2001factorgraphs}, and epistemic throughput studies
complete-record verification \citep{you2026epistemic}. BoE combines these
threads through signed claims shared across candidates and valued by final
selection.

\paragraph{Medical visual question answering.}
Medical VQA benchmarks span radiology, pathology, and biomedical figures, with large differences in answer space and difficulty. VQA-RAD introduced clinician-authored radiology questions; VQA-Med organizes modality, plane, organ, and abnormality questions; SLAKE adds semantic labels and a knowledge base \citep{lau2018vqarad,benabacha2019vqamed,liu2021slake}. PathVQA extends to pathology, PMC-VQA to large-scale literature figures, and MedXpertQA-MM to expert-level multimodal cases with rich clinical context \citep{he2021pathvqa,zhang2024pmcvqa,zuo2025medxpertqa}. Grounded benchmarks such as HEAL-MedVQA also measure whether answers use the relevant image region \citep{nguyen2025healmedvqa}. Medical answers often decompose into findings, locations, and attributes that can be checked separately. This is one domain-specific instantiation of a domain-independent partial-verification problem.
\section{Conclusion}

We introduced BoE, a test-time selection framework for
partial verification. BoE keeps the BoN candidate pool fixed, represents
reusable local claims with a signed candidate--factor graph, and allocates a
limited evidence budget according to estimated selection value. Its ideal
Bayesian formulation updates posterior mean utilities, while the practical
implementation uses discrimination-weighted score updates and a plug-in
value-of-information rule. Theoretically, we show that residual evidence
capacity limits evidence-driven improvement and that shared factor queries can
be substantially more query-efficient than candidate-level verification in a
factor-code model. Common-ledger experiments on medical VQA support the
proposed mechanism, showing that evidence is most useful when it is reliable,
contrastive, and relevant to the current decision. Together, these results
establish BoE as a structured extension of BoN for selection under partial
verification.

\paragraph{Limitations and future work.}
Although we have proven the effectiveness of the BoE algorithm, it remains constrained by the candidate generator and the evidence it produces. If the candidate pool contains no correct answer, or if extracted claims are noisy, redundant, or incorrectly grounded, evidence acquisition cannot recover the missing information. Future work should therefore post-train generators to produce more diverse candidate pools together with canonical, contrastive, and verifiable claims. The current controller also uses a myopic plug-in EVSI-per-cost rule, which may miss complementary evidence sequences or allocate budget suboptimally. Future work should explore learned or multi-step acquisition policies, stronger calibration and stopping rules, and optimization under measured inference costs. These two directions---improving evidence-producing models and improving evidence-allocation algorithms---are complementary paths toward a stronger BoE system.

\bibliographystyle{unsrt}
\bibliography{references}

\appendix
\clearpage
\section*{Appendix Contents}
\label{app:contents}
\addcontentsline{toc}{section}{Appendix Contents}

\noindent This appendix gives the details that are not needed for the main narrative but are useful for checking the claims. The appendix contains full proofs, the mathematical background behind the BoE, implementation details, metric definitions, additional diagnostics, and limitations.

\vspace{0.8em}
\noindent\hyperref[app:formal_theory]{Appendix~\ref*{app:formal_theory}: Formal Theory and Proofs}\dotfill\pageref{app:formal_theory}\\[0.25em]
\noindent\hyperref[app:experimental_details]{Appendix~\ref*{app:experimental_details}: Experimental Details and Additional Diagnostics}\dotfill\pageref{app:experimental_details}\\[0.25em]

\section{Formal Theory and Proofs}
\label{app:formal_theory}

This appendix formalizes the access model and BoE controllers introduced in
Sections~\ref{sec:problem}--\ref{sec:boe_method}, and then states and proves the
theoretical results summarized in Section~\ref{sec:theory}.

\subsection{Formal Partial-Verification Model}
\label{app:formal_problem}

The candidate-pool law, utility model, cheap context, signed incidence matrix,
and candidate--factor graph are defined in
Equations~\eqref{eq:candidate_sampling}--\eqref{eq:graph}. We give the formal
access and policy semantics here.

For a realized candidate pool, $\mathcal V_c$ indexes the $K$ candidate nodes
and $\mathcal V_f$ indexes the $M$ canonical claim nodes. The edge set
$\mathcal E_{\mathrm{cf}}$ consists of the pairs $(i,j)$ for which
$B_{ij}\ne0$, with sign $B_{ij}$. Thus, the identity of a claim node is
separated from candidate stance: candidates that assert and deny the same
canonical proposition share one node and differ only in their edge signs.

A factor action targets a claim node and returns one acquired outcome. That
outcome appears once in the transcript, although its effect may be applied to
several candidates through the incident signed edges. This propagation does
not create several independent observations. Candidate-level actions instead
inspect a complete response and need not act through a claim node.

The graph specifies the declared verification interface. It need not explain
every component of candidate utility and does not imply that
$\boldsymbol\Theta$ fully determines $\mathbf U$. Candidate utility may also
depend on unrepresented facts, interactions among claims, or aspects of a
response that are not locally verifiable.

Let $\mathcal A$ be the evidence-action menu. Each realized action
$e\in\mathcal A$ has cost $c(e)\ge c_{\min}>0$ and can be acquired at most
once; deliberate replicates must be represented as distinct menu items.
Equation~\eqref{eq:access_channel} separates view construction from the
returned outcome. Fixed model or tool parameters are suppressed, and the
executor seeds $\mathsf{Seed}_t^{\mathrm{exec}}$ are mutually independent and
independent of the task and controller randomization.

A deterministic controller is a measurable rule on finite histories that
returns either \textsc{stop} or an applicable, not-yet-acquired action. A
randomized controller may additionally use a private seed
$\mathsf{Seed}_\pi^{\mathrm{pol}}$ independent of the task and all executor
seeds. Starting from the empty history, the controller and executor induce the
acquired sequence and the random number $\tau$ of acquisitions before
\textsc{stop}.

On the event $\{\tau\ge t\}$, let
$T_t^\pi=(E_1,O_1,\ldots,E_t,O_t)$ be the random transcript prefix, with
$T_0^\pi=\varnothing$. The stopped history
$\overline T_t^\pi=T_{t\wedge\tau}^\pi$ is defined on the entire probability
space and remains equal to the terminal transcript after stopping. The
information available from the task and purchased outcomes after at most $t$
acquisitions is
\begin{equation}
\mathcal H_t=
\sigma(F,\overline T_t^{\pi}).
\end{equation}
For a deterministic policy, the next stop-or-acquire decision is
$\mathcal H_t$-measurable. For a randomized policy, it is measurable with
respect to
\[
\mathcal G_t^\pi
=
\mathcal H_t\vee\sigma(\mathsf{Seed}_\pi^{\mathrm{pol}}).
\]
Its support is restricted to \textsc{stop} and applicable, not-yet-acquired
actions, so $\tau$ is a stopping time relative to
$(\mathcal G_t^\pi)_t$.

Let $\Pi_C^{\mathrm{acc}}$ be the class of such access-admissible policies with
$\tau<\infty$ and realized cost at most $C$ almost surely. Their terminal
action--outcome transcript is
\begin{equation}
\label{eq:transcript}
T_\pi=(E_1,O_1,\ldots,E_\tau,O_\tau),
\qquad
c(T_\pi)=\sum_{t=1}^{\tau}c(E_t)\le C.
\end{equation}
The controller never observes $\Omega$, $\mathbf U$, the benchmark reference,
or an unpurchased outcome. Equation~\eqref{eq:policy_objective} follows from
the Bayes action for top-one selection. Its outer expectation averages over
task instances, candidate pools, evidence outcomes, stopping behavior, and
executor and policy randomization. If the supremum is attained,
$\pi_C^\star$ denotes a maximizing acquisition policy.

\subsection{Exact Bayesian and Ledger-Replay Controllers}
\label{app:boe_controller}

Consider a policy prefix on the event $\{\tau\ge t\}$. Let
$\mathsf t=((e_1,o_1),\ldots,(e_t,o_t))$ be a realization of
$T_t^\pi$, where $o_k\in\mathcal O_{e_k}$. Its realized cost is
\begin{equation}
\label{eq:cost_budget}
c(\mathsf t)
=
\sum_{k=1}^{t}c(e_k)
\le C.
\end{equation}
We write $\mathsf t\oplus(e,o)$ for sequence concatenation and let
$\mathcal A_{\mathrm{avail}}(f,\mathsf t)\subseteq\mathcal A$
contain the applicable actions that have not yet been acquired.
Current-prefix statements are understood for $t$ such that
$\Prob(\tau\ge t)>0$, under the corresponding conditional law.

\paragraph{Exact Bayesian controller.}
Under a specified observation model,
Equation~\eqref{eq:posterior_utility} defines the exact posterior mean.
Access-admissibility implies that the policy identity and its independent
randomization reveal no task information beyond $F$ and the realized
action--outcome history. The policy and prefix-length superscripts can
therefore be suppressed in the notation for $\mu_i$.

The Bayes selection is
\begin{equation}
\label{eq:selection_rule}
i^\star(f,\mathsf t)
\in
\argmax_{1\le i\le K}
\mu_i(f,\mathsf t),
\end{equation}
with a fixed deterministic tie-breaking rule. Its current value is
\begin{equation}
\label{eq:decision_value}
V_{\mathrm{sel}}(f,\mathsf t)
=
\max_i\mu_i(f,\mathsf t).
\end{equation}
For the empty transcript,
$\mu_i(f,\varnothing)=\mu_i^{\mathrm{prior}}(f)$.

The exact EVSI in Equation~\eqref{eq:evsi} uses the prospective random
outcome $O_e$ and the history-conditioned predictive distribution
\[
P(O_e=o\mid F=f,T_t^\pi=\mathsf t).
\]
The corresponding one-step greedy rule is
\begin{equation}
\label{eq:evsi_cost}
e^\star
\in
\argmax_{
e\in\mathcal A_{\mathrm{avail}}(f,\mathsf t):
\,c(\mathsf t)+c(e)\le C
}
\frac{
\EVSI_{\mathrm{sel}}(e\mid f,\mathsf t)
}{
c(e)
}.
\end{equation}
This is a sequential approximation to
Equation~\eqref{eq:policy_objective}, not a globally optimal acquisition
policy in general. The procedure stops when no available action fits the
remaining budget or when the largest value density is at most the threshold
$\eta$.

\paragraph{Ledger-replay approximation.}
The ledger replay replaces the exact observation model with pooled empirical
components. Each action $e$ has a categorical outcome alphabet
$\mathcal O_e$, and
\begin{equation}
\alpha_{ie}:
\mathcal O_e
\longrightarrow
\{-1,0,+1\}
\end{equation}
records whether an outcome supports, contradicts, or supplies no applicable
signal for candidate $i$.

For a factor action targeting claim $j$, this map combines the returned
verdict with the graph sign $B_{ij}$. A confirmation uses the direction
$B_{ij}$, a rejection reverses that direction, and an unresolved outcome or
an unrelated candidate gives zero. The same factor outcome can therefore
produce positive, negative, and zero updates across the candidate pool.
Candidate-level actions define $\alpha_{ie}(o)$ directly for the inspected
response.

\begin{algorithm}[H]
\caption{Best-of-Evidence score-based selection}
\label{alg:boe}
\begin{algorithmic}[1]
\State \textbf{Input:} executor holding raw $x$, candidate count $K$,
evidence budget $C$, action menu $\mathcal A$, pooled outcome masses
$\widehat p$, discrimination indices $\kappa$, threshold $\eta$
\State Sample
$\mathbf y_{1:K}\sim p_\phi^{\mathrm{gen}}(\cdot\mid x)$
and expose the realized cheap context $f$
\State Extract canonical claims and signed candidate stances; build
$G_{\mathrm{cf}}$
\State Instantiate evidence actions with their sources, costs, and
calibration groups
\State Initialize
$\mathsf t\gets\varnothing$
and evidence log
$\mathcal L\gets\varnothing$
\While{a feasible unrevealed action remains}
    \State Compute
    $\widehat{\EVSI}_{\mathrm{sel}}(e\mid f,\mathsf t)/c(e)$
    for every feasible action $e$
    \State Let $\widehat e^\star$ be the fixed-tie maximizer of the plug-in
    value density
    \If{
    $\widehat{\EVSI}_{\mathrm{sel}}
    (\widehat e^\star\mid f,\mathsf t)/
    c(\widehat e^\star)
    \le\eta$
    }
        \State \textbf{break}
    \EndIf
    \State Execute $\widehat e^\star$ and observe
    $o^\star\in\mathcal O_{\widehat e^\star}$
    \State Set
    $\mathsf t\gets
    \mathsf t\oplus(\widehat e^\star,o^\star)$
    \State Update all affected candidate scores using
    Equation~\eqref{eq:score_update}
    \State Append the acquired action and outcome to $\mathcal L$
\EndWhile
\State \textbf{Return} the fixed-tie maximizer of
$\widetilde s_i(\mathsf t)$, the score vector
$\widetilde{\mathbf s}(\mathsf t)$, and the evidence log $\mathcal L$
\end{algorithmic}
\end{algorithm}

Let $\gamma(e)$ be the dataset--channel calibration group assigned to action
$e$. The replay uses the pooled categorical mass
\[
\widehat p_e(o)
:=
\widehat p_{\gamma(e)}(o)
\]
and the assigned selection-side discrimination index
\[
\kappa(e)
:=
\kappa_{\gamma(e)}
\in(0,1).
\]
Either quantity may be estimated from the common ledger or fixed by the
protocol; neither is a parameter of a fitted joint observation model. Define
\begin{equation}
\label{eq:score_weight}
\operatorname{wt}(e)
=
\operatorname{logit}\kappa(e)
=
\log\frac{\kappa(e)}{1-\kappa(e)}.
\end{equation}

Candidate scores are initialized by
$\widetilde s_i(\varnothing)=s_i^{(0)}$, and
\[
\widetilde{\mathbf s}(\mathsf t)
:=
\bigl(
\widetilde s_1(\mathsf t),
\ldots,
\widetilde s_K(\mathsf t)
\bigr).
\]
All replay scores used in the reported experiments lie strictly between zero
and one, so their logits are finite. Suppressing the fixed instance context
$f$, Equation~\eqref{eq:score_update} applies the signed
discrimination-weighted increment to every affected candidate. For a factor
action, the acquired outcome and channel weight are shared, while
$\alpha_{ie}(o)$ determines the candidate-specific update direction.

For each possible outcome $o\in\mathcal O_e$, the replay computes the
hypothetical updated scores and applies
Equation~\eqref{eq:plugin_evsi}. Unlike the exact predictive distribution in
Equation~\eqref{eq:evsi}, $\widehat p_e$ is pooled at the
dataset--channel level and is not conditioned on the current history.
Consequently, $\widehat{\EVSI}_{\mathrm{sel}}$ is a plug-in action-ranking
score and need not inherit exact Bayesian semantics.

The implemented greedy controller ranks all feasible actions by
\[
\widehat{\EVSI}_{\mathrm{sel}}(e\mid f,\mathsf t)/c(e),
\]
selects a fixed-tie maximizer $\widehat e^\star$, and applies the same budget
and threshold stopping rules as the exact controller. The hat distinguishes
this implemented action from the exact-EVSI choice $e^\star$ in
Equation~\eqref{eq:evsi_cost}.

Likewise, $\kappa(e)$ measures candidate discrimination rather than
factor-truth accuracy or a likelihood parameter. We reserve ``posterior'' and
``EVSI'' without hats for the exact quantities in
Equations~\eqref{eq:posterior_utility} and \eqref{eq:evsi}.

The evidence log $\mathcal L$ records each acquired action, source, outcome,
cost, assigned discrimination index, and affected candidates. A shared factor
outcome is stored once. Localization metadata such as boxes, masks, and crops
determines the view supplied to an action but does not enter the score update
as an independent observation.

\subsection{Decision-Effective Signed Reuse and Local Channel Dominance}

Fix a policy $\pi$ and an acquisition round $t$ with
$\Prob(\tau\ge t)>0$. Let $h=(f,\mathsf t)$ be a
$P_{(F,T_t^\pi)\mid\tau\ge t}$-almost-everywhere realization of the active
history. Let
$\mathbf B\in\{-1,0,+1\}^{K\times M}$ be the signed candidate--claim
incidence matrix and let $\mathbf b_j$ denote its $j$-th column.

\begin{definition}[Effective reuse]
\label{def:effective_reuse}
Let $\mathbf W_h\succeq0$ be a residual decision-weight matrix at history
$h$. The effective reuse of claim $j$ is
\begin{equation}
\label{eq:effective_reuse}
R_{\mathrm{eff}}(j;h)
=
\mathbf b_j^\top
\mathbf W_h
\mathbf b_j.
\end{equation}
\end{definition}

The matrix $\mathbf W_h$ gives greater weight to uncertain candidates near the
current decision boundary and may suppress common-mode directions. Thus, the
number of candidates touched by a claim is not itself a measure of useful
reuse.

\begin{assumption}[Local quadratic decision model]
\label{ass:local_quadratic}
Around history $h$, the decision value produced by a small candidate-score
log-odds perturbation $\Delta\boldsymbol\ell$ has the local form
\begin{equation}
\label{eq:local_quadratic_value}
\Delta V
\approx
\frac{1}{2}
\Delta\boldsymbol\ell^{\top}
\mathbf W_h
\Delta\boldsymbol\ell,
\qquad
\mathbf W_h\succeq0.
\end{equation}
For a factor action $e$ targeting claim $j(e)$, let
$\Delta\boldsymbol\ell_e=\xi_e\mathbf b_{j(e)}$, where
$\E[\xi_e\mid h]=0$, $\E[\xi_e^2\mid h]=\iota_e(h)$, and the action cost is
$c(e)$. For a candidate-level check on candidate $i$, let
$\Delta\boldsymbol\ell_i^{\mathrm{cand}}=\zeta_i\mathbf e_i$, where
$\E[\zeta_i\mid h]=0$,
$\E[\zeta_i^2\mid h]=\iota_i^{\mathrm{cand}}(h)$, and the cost is
$c_i^{\mathrm{cand}}$.
\end{assumption}

\begin{proposition}[Local channel dominance]
\label{prop:channel_dominance}
Under Assumption~\ref{ass:local_quadratic}, the expected local decision-value
density of factor action $e$ is proportional to
\begin{boeanchor}{Formal local comparison}
\begin{equation}
\label{eq:factor_density}
D_{\mathrm{factor}}(e;h)
=
\frac{
\iota_e(h)R_{\mathrm{eff}}(j(e);h)
}{
c(e)
}.
\end{equation}
The density of a candidate-level check on candidate $i$ is proportional to
\begin{equation}
\label{eq:candidate_density}
D_{\mathrm{cand}}(i;h)
=
\frac{
\iota_i^{\mathrm{cand}}(h)
\mathbf e_i^\top
\mathbf W_h
\mathbf e_i
}{
c_i^{\mathrm{cand}}
}.
\end{equation}
Hence, within the local model, factor evidence dominates candidate-level
checking when
\begin{equation}
\label{eq:dominance_condition}
\max_{e\in\mathcal A_{\mathrm{fac}}(h)}D_{\mathrm{factor}}(e;h)
>
\max_{1\le i\le K}D_{\mathrm{cand}}(i;h).
\end{equation}
\end{boeanchor}
\end{proposition}

\begin{proof}
For factor action $e$, substitute
$\Delta\boldsymbol\ell_e=\xi_e\mathbf b_{j(e)}$ into
Equation~\eqref{eq:local_quadratic_value}:
\begin{align}
\E[\Delta V_e\mid h]
&\approx
\frac{1}{2}
\E[\xi_e^2\mid h]
\mathbf b_{j(e)}^\top
\mathbf W_h
\mathbf b_{j(e)}
\\
&=
\frac{1}{2}
\iota_e(h)R_{\mathrm{eff}}(j(e);h).
\end{align}
Dividing by $c(e)$ yields
$\frac{1}{2}D_{\mathrm{factor}}(e;h)$.

Similarly,
\begin{align}
\E[\Delta V_i^{\mathrm{cand}}\mid h]
&\approx
\frac{1}{2}
\E[\zeta_i^2\mid h]
\mathbf e_i^\top
\mathbf W_h
\mathbf e_i
\\
&=
\frac{1}{2}
\iota_i^{\mathrm{cand}}(h)
\mathbf e_i^\top
\mathbf W_h
\mathbf e_i.
\end{align}
Dividing by $c_i^{\mathrm{cand}}$ gives
$\frac{1}{2}D_{\mathrm{cand}}(i;h)$. The common factor $1/2$ cancels when the
two action classes are compared.
\end{proof}

This proposition is a local explanatory approximation to value per cost, not
a universal ordering of channels and not the plug-in controller objective in
Section~\ref{sec:boe_method}.

\subsection{Compressive Verification and Finite Probe Calculations}

\begin{assumption}[Factor-code candidate family]
\label{ass:factor_code}
Let $\boldsymbol\Theta\in\{0,1\}^M$ be uniformly distributed. The candidate
pool contains one candidate $\mathbf y_{\mathbf a}$ for every assignment
$\mathbf a\in\{0,1\}^M$, so $K=2^M$. Candidate utility is
\begin{equation}
\label{eq:factor_code_utility}
U_{\mathbf a}
=
\ind\{\mathbf a=\boldsymbol\Theta\}.
\end{equation}
To embed the construction in the access model, the executor holds the record
$X=\boldsymbol\Theta$ and the selector starts from constant $F$. A factor
query at coordinate $j$ and a candidate query at assignment $\mathbf a$ have
outcomes
\[
O_{e_j^{\mathrm{fac}}}=X_j,
\qquad
O_{e_{\mathbf a}^{\mathrm{cand}}}
=
\ind\{\mathbf a=X\}.
\]
Both query types inspect admissible views of the same executor-held record and
neither reads a utility variable or benchmark reference.
\end{assumption}

\begin{theorem}[Compressive-verification separation]
\label{thm:compressive_separation}
Under Assumption~\ref{ass:factor_code}, factor-level evidence identifies the
correct candidate with $M=\log_2K$ queries and succeeds with probability one.
Any policy that makes $q$ candidate-level queries and then outputs one
candidate succeeds with probability at most
\begin{equation}
\label{eq:candidate_query_bound}
\min\left\{\frac{q+1}{K},1\right\}.
\end{equation}
Consequently, success probability at least $1-\varepsilon$ requires at least
$\left\lceil(1-\varepsilon)K\right\rceil-1$ candidate-level queries.
\end{theorem}

\begin{proof}
Querying all $M$ coordinates reveals $\boldsymbol\Theta$ exactly and therefore
identifies $\mathbf y_{\boldsymbol\Theta}$. For candidate-level verification,
repeated queries are never useful. Consider first $q<K$ distinct queried
assignments. The probability that one of them is correct is $q/K$. If all
queries return zero, the posterior is uniform over the remaining $K-q$
assignments, so the best final guess succeeds with conditional probability
$1/(K-q)$. The total success probability is
\[
\frac{q}{K}
+
\frac{K-q}{K}\frac{1}{K-q}
=
\frac{q+1}{K}.
\]
For $q\ge K$, success is at most one. Solving
$(q+1)/K\ge1-\varepsilon$ gives the stated integer query lower bound.
\end{proof}

For direct comparison with a fixed query budget, the same candidate-query
bound can be written as
\begin{equation}
\label{eq:candidate_query_bound_main}
\min\!\left\{\frac{q+1}{K},1\right\}.
\end{equation}
The separation is a witness for possible compression through shared
coordinates; it does not assert that real VLM candidate pools form complete
binary codes.

\paragraph{Exact factor-code curves.}
For $K=2^M$, after $q\le M$ exact factor queries,
$2^{M-q}$ assignments remain possible. The optimal success probability is
\begin{equation}
\label{eq:factor_probe_success}
P_{\mathrm{factor}}(q)
=
\min\left\{\frac{2^q}{K},1\right\}.
\end{equation}
The candidate-level curve follows from
Theorem~\ref{thm:compressive_separation}:
\begin{equation}
\label{eq:candidate_probe_success}
P_{\mathrm{candidate}}(q)
=
\min\left\{\frac{q+1}{K},1\right\}.
\end{equation}
For $K=64$ and $M=6$, the first curve reaches one after six checks, whereas
the second requires 60 checks to exceed $95\%$ success.

\paragraph{Noisy-factor curve.}
Consider a separate model with $K=16$, $M=4$, and independent binary
symmetric channel observations
\[
O_j^{\mathrm{BSC}}
=
\Theta_j\oplus N_j,
\]
where $N_j\sim\operatorname{Bernoulli}(1-\rho)$ independently across $j$ and
independently of $\boldsymbol\Theta$. Thus
\[
\rho
=
\Prob(O_j^{\mathrm{BSC}}=\Theta_j)
\in[1/2,1].
\]
The full noisy transcript is
\[
T_{\mathrm{BSC}}
=
(O_1^{\mathrm{BSC}},\ldots,O_M^{\mathrm{BSC}}).
\]
One uniform bit sent through this channel contributes $1-H_b(\rho)$ bits, so
independence gives
\[
\MI_2(\boldsymbol\Theta;T_{\mathrm{BSC}})
=
M[1-H_b(\rho)].
\]
Because the one-hot utility vector $\mathbf U$ uniquely identifies
$\boldsymbol\Theta$ in this construction,
\begin{equation}
\label{eq:finite_capacity}
\MI_2(\mathbf U;T_{\mathrm{BSC}})
=
M[1-H_b(\rho)].
\end{equation}
The posterior mode is the observed code and is correct exactly when all
$M$ observed bits are correct. Relative to uniform-prior success $1/K$, its
gain is
\begin{equation}
\label{eq:finite_gain}
\Delta_{\mathrm{mode}}(\rho)
=
\rho^M-\frac{1}{K}.
\end{equation}
Both expressions vanish at $\rho=1/2$. The parameter $\rho$ here is a true
channel correctness probability and is distinct from the empirical
selection-discrimination index $\kappa$ in
Section~\ref{sec:boe_method}.

\subsection{Residual Evidence Capacity and Information Upper Laws}

Before round $t$, an access-admissible policy observes only the cheap context
and its purchased action--outcome history. Let $\Pi_C^{\mathrm{acc}}$ denote
the policies defined in Section~\ref{sec:problem} whose realized cost is at
most $C$ almost surely.

\begin{definition}[Evidence capacity]
\label{def:evidence_capacity}
For an adaptive policy $\pi$, let
\[
T_\pi=(E_1,O_1,\ldots,E_\tau,O_\tau)
\]
be its terminal action--outcome transcript. The residual evidence capacity is
\begin{equation}
\label{eq:evidence_capacity}
\Lambda_C
=
\sup_{\pi\in\Pi_C^{\mathrm{acc}}}
\MI(\mathbf U;T_\pi\mid F).
\end{equation}
\end{definition}

Conditioning on $F$ removes information already exposed through the cheap
interface. The transcript includes the adaptively selected actions as well as
their outcomes, so the supremum accounts for overlap and redundancy. Because
conditional mutual information is nonnegative, $\Lambda_C\ge0$.

For the finite noisy transcript above, conversion from bits to nats gives
\[
(\ln2)\MI_2(\mathbf U;T_{\mathrm{BSC}})
\le
\Lambda_C
\]
whenever its four probes are access-admissible and fit the budget. Equality
holds when those four unit-cost probes are the entire feasible menu and
$C=4$.

The access semantics are essential. If the complete raw record were included
in $F$ and each action outcome were only a deterministic function of that
record plus independent noise, then the transcript would add no conditional
information about $\mathbf U$ beyond $F$. The capacity definition therefore
applies to the costed view-access model in
Equation~\eqref{eq:access_channel}.

\subsubsection{Total-information upper law}

\begin{assumption}[Bernoulli candidate model]
\label{ass:bernoulli_candidates}
Before observing cheap context or evidence, the utility vector
$\mathbf U=(U_1,\ldots,U_K)$ has independent coordinates
$U_i\sim\operatorname{Bernoulli}(p_U)$. The cheap context satisfies
\begin{equation}
\label{eq:cheap_information}
\MI(\mathbf U;F)
\le
KJ_F.
\end{equation}
\end{assumption}

Equivalently, the main-text parameterization records the same assumption as
\[
\MI(\mathbf U;F)\le KJ_F,
\qquad J_F\ge0.
\]

\begin{theorem}[Evidence-capacity upper law]
\label{thm:upper_law}
Under Assumption~\ref{ass:bernoulli_candidates}, let an access-admissible,
budget-feasible policy observe $\mathcal Z=(F,T_\pi)$ and select a subset
$S(\mathcal Z)\subseteq\{1,\ldots,K\}$ of size $m$. Define
\begin{equation}
\label{eq:throughput_tm}
T_m
=
\E\left[\sum_{i\in S(\mathcal Z)}U_i\right].
\end{equation}
If $J_F$ and $\Lambda_C$ are measured in nats, then
\begin{equation}
\label{eq:upper_law}
T_m
\le
mp_U+
\sqrt{\frac{m}{2}\left(KJ_F+\Lambda_C\right)}.
\end{equation}
If they are measured in bits, the square-root term is multiplied by
$\sqrt{\ln2}$.
\end{theorem}

\begin{proof}
Let $\mathcal S_m$ be the collection of all size-$m$ subsets of
$\{1,\ldots,K\}$. For fixed $s\in\mathcal S_m$, define
\[
X_s
=
\sum_{i\in s}(U_i-p_U).
\]
By Hoeffding's lemma,
\[
\log\E\exp(\lambda X_s)
\le
\frac{m\lambda^2}{8}
\qquad
\text{for all }\lambda\in\mathbb R,
\]
so $X_s$ is $m/4$-sub-Gaussian.

Let $\widehat S=S(\mathcal Z)$ and, for every $s$ with
$\Prob(\widehat S=s)>0$, define
\[
D_s
=
D_{\mathrm{KL}}
\left(
P_{\mathbf U\mid\widehat S=s}
\,\Vert\,
P_{\mathbf U}
\right).
\]
The variational representation of relative entropy gives, for any
$\lambda>0$,
\begin{align}
\E[X_s\mid\widehat S=s]
&\le
\frac{D_s+\log\E\exp(\lambda X_s)}{\lambda}
\\
&\le
\frac{D_s}{\lambda}+\frac{m\lambda}{8}.
\end{align}
Optimizing over $\lambda$ yields
\begin{equation}
\label{eq:conditional_selection_bias}
\E[X_s\mid\widehat S=s]
\le
\sqrt{\frac{mD_s}{2}}.
\end{equation}
Averaging over $\widehat S$ and applying Jensen's inequality,
\begin{align}
\E[X_{\widehat S}]
&\le
\sum_s\Prob(\widehat S=s)\sqrt{\frac{mD_s}{2}}
\\
&\le
\sqrt{
\frac{m}{2}
\sum_s\Prob(\widehat S=s)D_s
}
\\
&=
\sqrt{\frac{m}{2}\MI(\mathbf U;\widehat S)}.
\end{align}
Since $\widehat S$ is a function of $\mathcal Z$, data processing gives
\[
\MI(\mathbf U;\widehat S)
\le
\MI(\mathbf U;\mathcal Z).
\]
By the chain rule and Definition~\ref{def:evidence_capacity},
\begin{align}
\MI(\mathbf U;\mathcal Z)
&=
\MI(\mathbf U;F)
+
\MI(\mathbf U;T_\pi\mid F)
\\
&\le
KJ_F+\Lambda_C.
\end{align}
Finally,
\[
T_m
=
mp_U+\E[X_{\widehat S}]
\le
mp_U+
\sqrt{\frac{m}{2}\left(KJ_F+\Lambda_C\right)}.
\]
The proof uses nats. If information is measured in bits, multiplying it by
$\ln2$ produces the factor $\sqrt{\ln2}$.
\end{proof}

The theorem can equivalently separate the acquisition policy from its
terminal selector. For any $\pi\in\Pi_C^{\mathrm{acc}}$, let
$\delta_m^{\mathrm{sel}}$ map $(F,T_\pi)$ to the size-$m$ set
\[
\widehat{\mathcal I}_{\pi,m}
=
\delta_m^{\mathrm{sel}}(F,T_\pi),
\]
and define
\[
\mathrm{EU}_m(\pi,\delta_m^{\mathrm{sel}})
=
\E\!\left[
\sum_{i\in\widehat{\mathcal I}_{\pi,m}}U_i
\right].
\]
Then the main-text form is
\begin{equation}
\label{eq:upper_law_main}
\mathrm{EU}_m(\pi,\delta_m^{\mathrm{sel}})
\le
mp_U+
\sqrt{\frac{m}{2}\left(KJ_F+\Lambda_C\right)}.
\end{equation}
This is a no-free-lunch upper bound; it does not assert that BoE attains the
right-hand side.

\subsubsection{Evidence-only gain}

\begin{assumption}[Conditional sub-Gaussian posterior]
\label{ass:conditional_subgaussian}
For every $F=f$ and every fixed subset $s$ of size $m$,
\[
X_s^{(f)}
=
\sum_{i\in s}
\left(
U_i-\E[U_i\mid F=f]
\right)
\]
is $m/4$-sub-Gaussian under $P_{\mathbf U\mid F=f}$. This holds, for
example, when the candidate utilities are conditionally independent Bernoulli
variables given $F=f$.
\end{assumption}

\begin{corollary}[Evidence-only gain bound]
\label{cor:evidence_only_gain}
Under Assumption~\ref{ass:conditional_subgaussian}, define
\[
V_m(F)
=
\max_{|s|=m}
\E\left[
\sum_{i\in s}U_i
\;\middle|\;
F
\right]
\]
and
\[
V_m(F,T_\pi)
=
\max_{|s|=m}
\E\left[
\sum_{i\in s}U_i
\;\middle|\;
F,T_\pi
\right].
\]
For every access-admissible, budget-feasible policy,
\begin{boeanchor}{Formal evidence-only ceiling}
\begin{equation}
\label{eq:evidence_only_gain}
\E\left[V_m(F,T_\pi)-V_m(F)\right]
\le
\sqrt{\frac{m}{2}\Lambda_C}.
\end{equation}
\small
The displayed form uses nats. In bits, the right-hand side is multiplied by
$\sqrt{\ln2}$.
\end{boeanchor}
\end{corollary}

\begin{proof}
Fix $F=f$ and let
\[
S^\star(f,T_\pi)
\in
\argmax_{|s|=m}
\E\left[
\sum_{i\in s}U_i
\;\middle|\;
F=f,T_\pi
\right].
\]
Write
\[
\mu_s(f)
=
\E\left[
\sum_{i\in s}U_i
\;\middle|\;
F=f
\right].
\]
Because $\mu_{S^\star}(f)\le\max_s\mu_s(f)=V_m(f)$,
\begin{align}
&\E\left[
V_m(f,T_\pi)-V_m(f)
\;\middle|\;
F=f
\right]
\\
&\le
\E\left[
\sum_{i\in S^\star}U_i-\mu_{S^\star}(f)
\;\middle|\;
F=f
\right].
\end{align}
Conditionally on $F=f$, the information-selection argument from
Theorem~\ref{thm:upper_law} gives
\begin{align}
&\E\left[
V_m(f,T_\pi)-V_m(f)
\;\middle|\;
F=f
\right]
\\
&\le
\sqrt{
\frac{m}{2}
\MI(\mathbf U;S^\star\mid F=f)
}
\\
&\le
\sqrt{
\frac{m}{2}
\MI(\mathbf U;T_\pi\mid F=f)
},
\end{align}
where the second inequality follows by conditional data processing because
$S^\star$ is a function of $(f,T_\pi)$.

Averaging over $F$ and applying Jensen's inequality,
\begin{align}
\E[V_m(F,T_\pi)-V_m(F)]
&\le
\E_F
\sqrt{
\frac{m}{2}
\MI(\mathbf U;T_\pi\mid F=f)
}
\\
&\le
\sqrt{
\frac{m}{2}
\MI(\mathbf U;T_\pi\mid F)
}
\\
&\le
\sqrt{\frac{m}{2}\Lambda_C}.
\end{align}
The bit-valued form follows by multiplying mutual information by $\ln2$.
\end{proof}

For the terminal-history notation used in Section~\ref{sec:theory}, define
\[
V_m(f)
=
\max_{\substack{\mathcal J\subseteq\{1,\ldots,K\}\\|\mathcal J|=m}}
\E\!\left[
\sum_{i\in\mathcal J}U_i
\mid F=f
\right]
\]
and, for $P_{F,T_\pi}$-almost every $(f,\mathsf t)$,
\[
V_m^\pi(f,\mathsf t)
=
\max_{\substack{\mathcal J\subseteq\{1,\ldots,K\}\\|\mathcal J|=m}}
\E\!\left[
\sum_{i\in\mathcal J}U_i
\mid F=f,T_\pi=\mathsf t
\right].
\]
Writing $V_m(F)$ and $V_m^\pi(F,T_\pi)$ for the corresponding random
evaluations, the top-one notation of Section~\ref{sec:boe_method} satisfies
\[
V_{\mathrm{sel}}(f,\mathsf t)
=
V_1^\pi(f,\mathsf t)
\]
at a terminal transcript. Equation~\eqref{eq:main_evidence_bound} is the
corresponding two-step information ceiling.

\subsection{Relation to Candidate-Level JaKoB Verification}

Suppose every action verifies one complete candidate, each action has unit
cost, and each nonredundant verification contributes the same conditional
information $I_{\mathrm{ver}}$. A budget of $B$ such actions gives
\begin{equation}
\label{eq:homogeneous_capacity}
\Lambda_C
=
BI_{\mathrm{ver}}.
\end{equation}
With the normalization $I_{\mathrm{ver}}=1$, the homogeneous JaKoB form becomes
\begin{equation}
\label{eq:jakob_capacity_special}
Bp+\sqrt{JKB}
=
p\Lambda_C+\sqrt{JK\Lambda_C}.
\end{equation}
This identity relies on homogeneous, nonredundant, candidate-level checks.
General BoE actions can have unequal costs, unequal discrimination,
overlapping information, and shared signed effects. The product form is
therefore used only as a candidate-level special-case guide, not as a general
law for reusable local evidence.

\section{Experimental Details and Additional Diagnostics}
\label{app:experimental_details}

\subsection{Common Ledger and Policy Definitions}

For each question, candidate generation and all potential evidence calls are
performed once. The resulting offline ledger stores the candidate pool,
canonical claims, signed candidate--factor graph, channel identities, action
costs, evidence outcomes, and evaluation utilities. Evaluation utilities are
hidden from deployable replay policies; they are used only for scoring and for
the label-guided diagnostic. The evidence judge returns observations and does
not rank candidates.

The main $C=16$ comparison contains five policies. Raw BoN returns the
zero-evidence majority winner. Random factor acquisition uses the same factor
representation and score update as BoE but samples only factor actions.
Whole-response judging purchases complete-candidate judgments. BoE ranks all
available actions by the plug-in value in
Equation~\eqref{eq:plugin_evsi}. The myopic label-guided allocator uses
evaluation labels at each step to reveal the factor that most improves the
current selection. It is non-deployable, factor-only, and neither globally
optimal nor an upper bound.

The main BoE menu includes whole-response actions, whereas random acquisition
is factor-only. Their difference is therefore an unmatched-menu descriptive
policy gap. A factor-only BoE replay on MedXpertQA-MM, reported below, matches
the random baseline's action menu.

All four main cells use Qwen3-VL-30B-A3B for candidate generation and
Qwen3-VL-235B-A22B for evidence judgment. We draw $K=16$ candidates at
temperature $1.1$ and replay budgets $C\in\{1,2,4,8,16\}$. The costs
$1$, $4$, and $8$ are synthetic design costs. We did not measure wall-clock,
token, or monetary costs, and the retained package does not contain a
cost-ratio sensitivity analysis.

\subsection{Dataset Protocols}

\paragraph{SLAKE.}
The SLAKE cell contains 645 open-answer questions and uses the earlier
8B-generator/30B-judge free-factor pipeline. It does not use the fixed
five-slot battery and is retained only as a historical comparison; it is not
part of the four-cell main benchmark.

\paragraph{VQA-Med.}
VQA-Med provides modality, plane, organ, and abnormality questions. Because
the available local copy contains only the training partition, we construct a
deterministic evaluation subset of 2,334 questions. The structured modality,
anatomy, and view vocabularies make this the most verifier-rich open-answer
cell.

\paragraph{PathVQA.}
PathVQA contributes 9,903 open pathology questions after removing binary
yes/no examples. Pathology images frequently do not support radiology-oriented
view or grading slots, so those slots often return $\varnothing$ or have weak
fitted discrimination. This cell measures transfer of the radiology-oriented
battery to pathology images.

\paragraph{PMC-VQA.}
We sample 10,000 questions and hide the original multiple-choice options from
the generator. The correct option text becomes the open-answer reference.
Items whose reference is an option-specific placeholder such as ``none of the
above'' or ``cannot determine'' are filtered. This protocol removes the
option-term evidence channel and is not an official open-answer split.

\paragraph{MedXpertQA-MM.}
We use the 2,000-question multimodal test set. Every question is a
single-answer, five-way MCQ with between one and six images. Each factor carries
a figure identifier; grounded factor actions inspect the corresponding image,
whereas the whole-response judge receives all images. Candidate correctness is
option-count-aware exact matching over A--E. Raw BoN selects an incorrect
candidate on 1,333 questions.

One MedXpertQA-MM record has no valid parsed candidate. The raw policy still
selects its fallback candidate, which is evaluated as incorrect, so this item
is included in the policy-aligned \textsc{MajorityWrong} set. We use 1,333
throughout because it exactly matches the failures of the evaluated raw policy.

\begin{table}[H]
\centering
\small
\caption{\textbf{Detailed candidate-generation statistics.}}
\label{tab:generation_health_appendix}
\resizebox{\linewidth}{!}{
\begin{tabular}{lrrrrrr}
\toprule
Dataset
& Questions
& $K$
& Candidate rows
& Parsed (\%)
& Answer (\%)
& Oracle@$K$ \\
\midrule
SLAKE           & 645    & 16 & 10,320  & 90.5 & 98.7 & 0.588 \\
VQA-Med         & 2,334  & 16 & 37,344  & 95.1 & 97.7 & 0.821 \\
PathVQA         & 9,903  & 16 & 158,448 & 94.0 & 99.1 & 0.615 \\
PMC-VQA         & 10,000 & 16 & 160,000 & 88.2 & 98.5 & 0.651 \\
MedXpertQA-MM   & 2,000  & 16 & 32,000  & 96.2 & 94.5 & 0.647 \\
\bottomrule
\end{tabular}}
\end{table}

\subsection{Evidence Channels}

The open-answer battery contains modality, anatomical region, view or plane,
primary finding, and one answer-relevant attribute. The value $\varnothing$
is a first-class outcome: it marks an inapplicable or unresolved slot and
contributes zero score increment. Canonicalized claims retain candidate stance.
For MedXpertQA-MM, a factor also records its figure identifier.

\begin{table}[H]
\centering
\small
\caption{\textbf{Evidence channels used in the current experiments.}}
\label{tab:evidence_channel_costs}
\begin{tabular}{llll}
\toprule
Channel & Typical observation & Cost & Score role \\
\midrule
Structured / metadata
& Vocabulary, option, or metadata check
& 1
& Channel discrimination weight \\
VLM factor
& Whole-image factor verdict
& 4
& Channel discrimination weight \\
Grounded VLM
& Figure-specific factor verdict
& 4
& Grounded discrimination weight \\
Whole response
& Complete-candidate verdict
& 8
& Candidate discrimination weight \\
\bottomrule
\end{tabular}
\end{table}

Grounding metadata, masks, bounding boxes, and crop quality determine what the
judge sees but do not enter the selection score as independent evidence.
Grounded and whole-image VLM judgments remain separate calibration groups.

\subsection{Open-Answer Evaluation and Channel Calibration}

For open answers, deterministic evaluation first applies normalized string
matching, substring matching, and numerical tolerance. Remaining candidate
answers are graded for semantic equivalence against the reference answer.
Per-example evaluation utilities are not revealed during deployable policy
replay. The open-answer evaluator and evidence judge use the same model tier,
so the open-answer cells do not provide an independent-evaluator test;
whole-response results are consequently diagnostic. MedXpertQA-MM uses exact
A--E matching and does not require an LLM correctness evaluator.

For dataset--channel group $g$, the selection-side discrimination index is
\begin{equation}
\label{eq:kappa_calibration_appendix}
\kappa_g
=
0.5+
\frac{1}{2}
\left[
\Prob(U=1\mid\mathrm{support},g)
-
\Prob(U=1\mid\mathrm{contradict},g)
\right].
\end{equation}
This quantity is neither per-factor truth accuracy nor a likelihood-ratio
parameter. It measures whether support from group $g$ makes an affected
candidate more likely to be correct. An index of $0.5$ yields zero score
weight, while an index below $0.5$ reverses the apparent update direction.
Each action inherits $\kappa(e)=\kappa_{\gamma(e)}$.

The fitted indices and categorical outcome frequencies are estimated from the
same ledger used for replay. The latter define
$\widehat p_e(o)=\widehat p_{\gamma(e)}(o)$ in
Equation~\eqref{eq:plugin_evsi}. The explicitly fixed Cheap values are not
fitted. This protocol supports a mechanism analysis, not a held-out
calibration estimate; deployment-oriented evaluation would require a separate
calibration split or cross-fitting.

\begin{table}[H]
\centering
\small
\caption{\textbf{Complete assigned channel-discrimination index table.}
The $s1$, $s3$, and $s5$ columns are populated fixed-battery structured
channels. ``n/a'' means that the channel is absent. The starred Cheap value is
a hand-set index in the open-answer runs.}
\label{tab:full_reliability}
\resizebox{\linewidth}{!}{
\begin{tabular}{lrrrrrrr}
\toprule
Dataset
& Cheap $\kappa_g$
& VLM $\kappa_g$
& Grounded $\kappa_g$
& Verifier $s1$
& Verifier $s3$
& Verifier $s5$
& Whole $\kappa_g$ \\
\midrule
SLAKE
& $0.55^\ast$ & 0.593 & 0.612 & n/a   & n/a   & n/a   & 0.670 \\
VQA-Med
& $0.55^\ast$ & 0.539 & 0.505 & 0.624 & 0.522 & 0.536 & 0.715 \\
PathVQA
& $0.55^\ast$ & 0.518 & 0.534 & 0.517 & 0.310 & 0.275 & 0.608 \\
PMC-VQA
& $0.55^\ast$ & 0.502 & 0.522 & 0.534 & 0.512 & 0.377 & 0.623 \\
MedXpertQA-MM
& 0.505        & 0.514 & 0.526 & n/a   & n/a   & n/a   & 0.642 \\
\bottomrule
\end{tabular}}
\end{table}

The MedXpertQA-MM implementation carries unused hand-set defaults for battery
slots that do not exist in its MCQ protocol. We report these slots as ``n/a''
because they do not enter its score updates.

\subsection{Policy-Specific Score-Proxy Diagnostics}

Equation~\eqref{eq:sharpening_score} is evaluated separately for each replay
policy. Its columns are not additive channel components, and negative values
are permitted because the underlying candidate scores are uncalibrated.

\begin{table}[H]
\centering
\small
\caption{\textbf{Policy-specific entropy-shaped score proxy at $C=16$.}
Values are dimensionless means of
$\mathsf{Sharp}_{d,C}^{\pi}$ under separate replay policies. The
label-guided factor column is a non-deployable diagnostic and has no
upper-bound interpretation.}
\label{tab:full_sharpening}
\resizebox{\linewidth}{!}{
\begin{tabular}{lrrrrr}
\toprule
Dataset
& BoE policy
& Whole response
& Grounded factor
& Label-guided factor
& Random factor \\
\midrule
VQA-Med
& +1.025 & +0.653 & +0.011 & +0.193 & +0.280 \\
SLAKE
& +0.607 & +0.389 & $-0.003$ & +0.142 & +0.235 \\
MedXpertQA-MM
& +1.006 & +0.799 & +0.021 & +0.015 & +0.021 \\
PMC-VQA
& $-0.276$ & $-0.202$ & $-0.020$ & +0.027 & $-0.006$ \\
PathVQA
& $-0.737$ & $-0.617$ & $-0.116$ & +0.297 & +0.092 \\
\bottomrule
\end{tabular}}
\end{table}

On PathVQA, the label-guided factor diagnostic has a positive proxy value while
the realized BoE update does not. On MedXpertQA-MM, whole-response replay has
a large positive proxy value although its accuracy is nearly unchanged from
raw majority. These cases show why $\mathsf{Sharp}_{d,C}^{\pi}$ should not be
read as information or accuracy gain.

\subsection{Paired Inference and Denominator Integrity}

The main text reports question-level paired-bootstrap $95\%$ intervals and
exact two-sided McNemar tests for BoE versus random factor acquisition. The
retained summaries do not preserve the bootstrap resample count, interval
construction, or random seed. They are therefore conditional fixed-ledger
summaries rather than independently reproducible estimates. The analyses are
exploratory, uncorrected for multiple comparisons, condition on one
candidate-generation seed, and omit candidate-pool and calibration-estimation
variability.

The full-set accuracy change is reconstructed from question-level transitions:
\begin{equation}
\label{eq:paired_accuracy_identity}
\Delta\mathrm{Acc}
=
\frac{
N(\text{raw wrong, BoE correct})
-N(\text{raw correct, BoE wrong})
}{N}.
\end{equation}
This identity keeps corrections and harmful flips on a common denominator;
subtracting separately normalized conditional rates is not an accuracy change.

\begin{table}[H]
\centering
\small
\caption{\textbf{Integer reconstruction of BoE--raw accuracy at $C=16$.}
Corrections are raw-wrong/BoE-correct transitions; harms are
raw-correct/BoE-wrong transitions.}
\label{tab:transition_identity}
\begin{tabular}{lrrrrrr}
\toprule
Dataset & $N$ & Raw correct & Raw wrong & Corrections & Harms
& $\Delta$Acc (pp) \\
\midrule
VQA-Med       & 2,334  & 1,478 & 856   & 38  & 32  & +0.26 \\
PathVQA       & 9,903  & 2,900 & 7,003 & 208 & 165 & +0.43 \\
PMC-VQA       & 10,000 & 3,641 & 6,359 & 240 & 182 & +0.58 \\
MedXpertQA-MM & 2,000  & 667   & 1,333 & 20  & 12  & +0.40 \\
\bottomrule
\end{tabular}
\end{table}

For MedXpertQA-MM, the 2,000 questions decompose into 667 raw-correct
questions, 627 raw-wrong questions whose pool contains a correct candidate,
and 706 raw-wrong questions whose pool does not. BoE repairs 20 of the 627
fixable failures and harms 12 of the 667 raw-correct questions, yielding
$(20-12)/2000=+0.40$ percentage points. Random factor, whole-response, and the
myopic label-guided allocator repair 9, 6, and 17 of the 1,333 raw failures,
respectively.

\subsection{Matched-Menu and Budget Diagnostics on MedXpertQA-MM}

The main BoE policy can purchase whole-response judgments, whereas random
acquisition is factor-only. A separate replay removes whole-response actions
from BoE so that the action menus match.

\begin{table}[H]
\centering
\small
\caption{\textbf{MedXpertQA-MM action-menu comparison at $C=16$.}
The menu-matched replay comparison is BoE factor-only versus random factor.}
\label{tab:medxpert_matched_menu}
\begin{tabular}{lrr}
\toprule
Policy & Accuracy (\%) & Gap from random factor (pp) \\
\midrule
Raw BoN          & 33.35 & --- \\
Random factor    & 33.40 & 0.00 \\
BoE factor-only  & 33.55 & +0.15 \\
BoE full menu    & 33.75 & +0.35 \\
\bottomrule
\end{tabular}
\end{table}

The available matched-menu gap is only three questions and has no reported
paired interval. It should therefore be treated as descriptive. The full-menu
gap cannot be interpreted as a pure routing effect.

The existing ledger also supports a budget sweep without regenerating
candidates. Table~\ref{tab:medxpert_budget} shows no monotone increase in the
BoE--random gap. Corrections and harms refer to BoE versus raw BoN.

\begin{table}[H]
\centering
\small
\caption{\textbf{MedXpertQA-MM fixed-ledger budget sweep.}
Accuracies and gaps are percentages and percentage points, respectively.}
\label{tab:medxpert_budget}
\begin{tabular}{rrrrrrrr}
\toprule
$C$ & Raw BoN & BoE & Random & BoE--random & Corrections & Harms
& BoE--raw \\
\midrule
1  & 33.35 & 33.35 & 33.30 & +0.05 & 5  & 5  & +0.00 \\
2  & 33.35 & 33.30 & 33.20 & +0.10 & 6  & 7  & $-0.05$ \\
4  & 33.35 & 33.55 & 33.20 & +0.35 & 9  & 5  & +0.20 \\
8  & 33.35 & 33.45 & 33.05 & +0.40 & 15 & 13 & +0.10 \\
16 & 33.35 & 33.75 & 33.40 & +0.35 & 20 & 12 & +0.40 \\
\bottomrule
\end{tabular}
\end{table}

Across $C\in\{1,2,4,8,16\}$, BoE--random remains at or below $0.40$ points
and does not grow monotonically. Together with the $C=16$ paired interval in
Table~\ref{tab:paired_uncertainty}, this does not support a stable
MedXpertQA-MM allocation gain.

\end{document}